%% file: paper.tex
\documentclass[sigconf]{acmart}

\usepackage{algorithm}
\usepackage{algpseudocode}
\usepackage{tikz}
\usepackage{placeins}
\usetikzlibrary{arrows,positioning,fit,shapes,calc,decorations.pathreplacing}
\usepackage{subcaption}
\usepackage{xspace}

\newcommand{\cev}[1]
{\reflectbox{\ensuremath{\vec{\reflectbox{\ensuremath{#1}}}}}}
\newcommand{\SYS}{DeepProbe\xspace}

\DeclareMathOperator*{\argmax}{arg\,max}

\copyrightyear{2017} 
\acmYear{2017} 
\setcopyright{acmcopyright}
\acmConference{KDD'17}{}{August 13--17, 2017, Halifax, NS, Canada.}
\acmPrice{15.00} 
\acmDOI{http://dx.doi.org/10.1145/3097983.3098148}
\acmISBN{978-1-4503-4887-4/17/08}
\begin{document}

\title[\SYS: Information Directed Sequence Understanding via RNN]{\SYS: Information Directed Sequence Understanding and Chatbot Design via Recurrent Neural Networks}
% , with Chatbot Application, emphasize in title?
\iffalse
\author{Anonymous, Anonymous, Anonymous}
\else
\author{Zi Yin}
\authornote{This work was completed during the first author's internship at Microsoft.}
\affiliation{%
       \institution{Stanford University}
       \streetaddress{350 Serra Mall}
       \city{Stanford}
       \state{CA}
       \postcode{94305}
}
\email{zyin@stanford.edu}

\author{Keng-hao Chang}
\affiliation{%
       \institution{Microsoft}
       \streetaddress{1020 Enterprise way}
       \city{Sunnyvale}
       \state{CA}
       \postcode{94084}
}
\email{kenchan@microsoft.com}

\author{Ruofei Zhang}
\affiliation{%
       \institution{Microsoft}
       \streetaddress{1020 Enterprise way}
       \city{Sunnyvale}
       \state{CA}
       \postcode{94084}
}
\email{bzhang@microsoft.com}
\fi

\input{abstract.tex}

\keywords{Deep Learning, RNN, Seq2Seq, ChatBot, Recommendation system, Attention Mechanism, Online Advertising, Sponsored Search, Query Rewriting, Probabilistic Scoring, Information Gain}

\maketitle

\input{intro.tex}

\input{model.tex}
\input{application.tex}

\input{conclusion.tex}

\bibliographystyle{ACM-Reference-Format}
\bibliography{learning}

\end{document}

%% file: abstract.tex
\begin{abstract}\label{abstract}

Information extraction and user intention identification are central topics in modern query understanding and recommendation systems. In this paper, we propose \SYS, a generic information-directed interaction framework which is built around an attention-based sequence to sequence (seq2seq) recurrent neural network. \SYS can rephrase, evaluate, and even actively ask questions, leveraging the generative ability and likelihood estimation made possible by seq2seq models. \SYS makes decisions based on a derived uncertainty (entropy) measure conditioned on user inputs, possibly with multiple rounds of interactions. Three applications, namely a rewritter, a relevance scorer and a chatbot for ad recommendation, were built around \SYS, with the first two serving as precursory building blocks for the third. We first use the seq2seq model in \SYS to rewrite a user query into one of standard query form, which is submitted to an ordinary recommendation system. Secondly, we evaluate \SYS's seq2seq model-based relevance scoring. Finally, we build a chatbot prototype capable of making active user interactions,  which can ask questions that maximize information gain, allowing for a more efficient user intention idenfication process. We evaluate first two applications by 1) comparing with baselines by BLEU and AUC, and 2) human judge evaluation. Both demonstrate significant improvements compared with current state-of-the-art systems, proving their values as useful tools on their own, and at the same time laying a good foundation for the ongoing chatbot application.

\end{abstract}

%% file: intro.tex
\section{Introduction}

Recent years have witnessed a boom in deep learning, which revolutionizes areas including computer vision, speech recognition and natural language processing. One widely-used deep learning model is the sequence to sequence (seq2seq) model, which demonstrated their power in machine translation \cite{sutskever2014sequence}, achieving higher BLEU score than conventional methods like phrased-based statistical machine translation models. %In computer vision, they have been used to generate visual attention \cite{mnih2014recurrent} and image captions \cite{xu2015show}. 

Different bells and whistles have been developed which further boost the performance of seq2seq models, one prominent example of which is the attention mechanism. In \cite{luong-pham-manning:2015:EMNLP}, the authors proposed this mechanism, which augments the top hidden vector at the decoder side with a weighted average of the encoder hidden vectors. The weights can be calculated through a cosine similarity or a generalized matrix inner product, where the weight matrix is part of the parameters to be learnt. By adding attention to the deep seq2seq model, the authors were able to better align inputs and outputs, and subsequently achieve an additional 5.0 improvement in BLEU score. %Similar attention implementations were discussed in \cite{xu2015show} for visual tasks.

On top of natural languages, seq2seq models can be trained with literally any kind of paired sequence data.  Authors in \cite{DBLP:journals/corr/VinyalsL15} build a model with IT helpdesk question-answer conversation log so that it will read new user questions and respond ``machine-translated'' answers. Authors in \cite{45189} take email correspondence log to build a model that suggests email reply candidates for users to choose from in mobile environment. Reflecting that a question-answering system would need a knowledge base to search answers from, we propose a staged approach that can leverage existing recommend system as is, serving as the knowledge base to search for the right answer. 
We apply the seq2seq model to understand user questions, using it to rewrite the question to one in standard query form that an ordinary recommendation system would understand. The rewrite is submitted to the recommendation system to retrieve a set of candidate answers. We show that with attention mechanism the model can rewrite questions with better quality measured by BLEU score. We also show that those rewrites can retrieve ads from a commercial search engine with better human labeled quality, proving that the system has significant commercial values. 
%This is in fact contradicting to other attempts using seq2seq for conversation \cite{DBLP:journals/corr/VinyalsL15}, in which they found attention mechanism is not helpful. We argue that attention mechanism was created to ``align'' machine translations, for which the alignment characteristics appear in question-rewriting but not question-answering.

%However, we believe this procedure has room for improvement. One key observation here is that many words are operational, and their effects are not always linear. For example, in the sequence ``not good'', the word ``not'' acts upon ``good'' and negates its meaning. This process should produce a hidden vector closer to ``bad''. In sentiment analysis literatures, similar problem was noticed and was approached using recursive neural networks, as proposed by Socher et. al. in \cite{socher2012semantic}, and recursive tensor network such as \cite{socher2013recursive}. We introduce a dynamic tensor attention network to address this problem, and combined it with Long-Short Term Memory (LSTM) networks to serve as a likelihood estimator.

Another powerful aspect of the seq2seq model besides its generativeness, namely its statistical property as likelihood estimators, have not been fully investigated by previous work. We built a seq2seq likelihood estimator in \SYS, which serves as the central model for an information directed evaluation and interaction framework. When used as a evaluation tool, a posterior probability derived from the seq2seq likelihood estimator will be calculated which serves as a relevance criterion. We can use it to refine candidates returned by a recommendation system. By comparing it against existing baselines like CDSSM \cite{shen2014learning}, we find significant performance improvement evaluated by AUC on a manually labeled dataset. When used as a interactive tool, the seq2seq model steers an agent through the user interaction process. An agent like a chatbot, tries to identify the intent of a user at a interactive session. The agent, using the seq2seq estimator, calculates the conditional entropy through a Naive Bayes procedure, which will be updated every time new information comes, i.e. a new user input. %It is the conditional entropy steering a interactive process: the agent first analyzes the information provided by the user input, and subsequently make a decision to either make a recommendation or ask further questions to gather more information. 
The agent iteratively uses the information to make a decision to either make a recommendation or ask further questions to gather more information. 
We build a chatbot prototype using this framework. The prototype is built on a commercial search engine which recommends product ads. The chatbot will recommend an ad if a user asks questions with product intent. When the user intention is not clear, it actively asks the user by formulating questions around product attributes that maximize the expected information gain.

The contribution of the paper is summarized as follows. We introduce \SYS, an information-directed interaction framework built upon a seq2seq model. We propose and implement a practical way to answer user questions in a staged approach: (1) we apply seq2seq model to understand and rewrite user questions into one that an ordinary recommendation system can understand and return candidates, (2) we use seq2seq model to score and pick better candidates, and finally (3) we use seq2seq to derive confidence measure and probe users for clarification if necessary.

%% file: model.tex
\section{Models}\label{model}

\subsection{Deep Multi-layer Seq2Seq Attention Model}\label{sec:seq2seq}
%When applied to translation tasks, for example in \cite{luong2014addressing}, the seq2seq models are used in a beam-search fashion; at each step, the most likely $k$ words were fed to the decoder as the next input until a $\langle$EOS$\rangle$ token is met, thus allowing it to generate target sequences. 
We use a seq2seq neural network enhanced with attention mechanism, which is illustrated in Figure \ref{fig:seq2seq}.
A seq2seq model is comprised of an encoder and a decoder, each consisting of several vertically stacked layers. Below we give a detailed explanation.

\subsubsection{Embedding Layer}
The embedding layer takes a word and converts it to its vector representation. The parameter required for this layer is a matrix $W_{emb}\in \mathbb R^{d_{emb}\times |\mathcal V|}$.  Specifically, when a word with index $i$ is given to the embedding layer, it produces $W_{\cdot, i}$, the $i$-th column of the matrix, which is a dimension $d_{emb}$ vector. We learn separate embedding layers and parameters for the encoder and decoder, i.e. two $W_{emb}$ matrices.

\subsubsection{Variable-depth LSTM Recurrent Layers}
The LSTM recurrent layer with depth $l$ consists of $l$ vertically stacked LSTM blocks. Each LSTM block takes three inputs: $e_t$, $c_{t-1}$ and $h_{t-1}$, where $e_t$ is the input from below, $c_{t-1}$ and $h_{t-1}$ are inputs from the previous step. Its output, $h_t$, is computed in the following way:
\begin{align*}
i_t&=\sigma (W_{ei}e_t+W_{hi}h_{t-1}+b_i)\\
f_t&=\sigma (W_{ef}e_t+W_{hf}h_{t-1}+b_f)\\
c_t&=f_t \cdot c_{t-1}+i_t \cdot \tanh(W_{ec}e_t+W_{hc}h_{t-1}+b_c)\\
o_t&=\sigma(W_{eo}e_t+W_{ho}h_{t-1+}b_o)\\
h_t&=o_t\cdot\tanh(c_t)
\end{align*}
where $\cdot$ denotes the element-wise product between vectors. LSTM is an enhanced recurrent neural network (RNN) that addresses short-term memory issue of a vanilla RNN, by maintaining additional cell vector $c_t$ and introducing input gate $i_t$, forget gate $f_t$, and output gate $o_t$. Detailed discussions of the advantages of LSTM can be found in \cite{lstm}, which is omitted in this paper due to the space limit. 
For the lowest LSTM layer, $e_t$ is the output of embedding layer with dimension $d_{emb}$, so $W_{e*}  \in R^{d_{h} \times d_{emb}}$,  $W_{h*} \in R^{d_{h} \times d_{h}}$, and $b_* \in R^{d_{h}}$ are the parameters to be learned. For the upper LSTM layers, $W_{e*}, W_{h*} \in R^{d_{h} \times d_{h}}$, and $b_* \in R^{d_{h}}$ are the parameters to be learned.
$\sigma(\cdot)$ here denotes sigmoid, a nonlinear activation function. For the encoder, each LSTM block is in fact bi-directional (BLSTM), it outputs concatenated hidden vectors from forward and backward directions, for which the final vector fed to decoder is from the last of both directions in concatenation form: $[\overrightarrow{h}_m;\overleftarrow{h}_{-1}]$.
For the decoder, each LSTM block has only forward direction, so readers should interpret $d_h$ accordingly, e.g. $d_h$ in decoder should be twice the size of that in encoder. In encoder each LSTM layer other than the lowest should reduce the input size by half so after concatenation the final output size of each BLSTM layer is the same.

\subsubsection{Attention Layer}

For every top hidden vector of the decoder, we augment it with an attention vector, $g_t$, which is obtained by combining the top hidden vectors from the encoder. The attention mechanism will be discussed in section \ref{sec:atten}. After concatenating the attention vector $g_t$ with the output vector of the top LSTM layer $h_t$, we apply a fully connected layer to reduce the dimension back to the same size as the input hidden vector:
$ 
\hat{h}_t = Relu(W_c[g_t;h_t]+b_c)
$, 
where for $Relu$ is the rectified nonlinearity unit, $max(0, \cdot)$. Here the parameters are $W_c \in R^{d_{h} \times 2d_{h}}$ and $b_c \in R^{d_{h}}$.
The output $\hat{h}_t$ will be passed to the next layer.

%Note to compare against base model without attention, the concatenation here artificially doubles the size of parameter $W_{proj}$, so we apply hidden layer to downsize the input for

\subsubsection{Projection Layer}
The projection layer takes the combined hidden and attention vector as input, and outputs a vector of dimension $|\mathcal V|$. Its parameters include a weight matrix $W_p\in\mathbb R^{|\mathbb V| \times d_h}$ and a bias vector $b_p\in \mathbb R^{|\mathcal V|}$. The output at step $t$ is computed as
$v_t=\text{softmax}(W_p\hat{h}_t+b_p)$.
Note $v$ is a non-negative vector which sums up to 1, hence it can be viewed as a distribution on the vocabulary $\mathcal V$. The likelihood of seeing a specific word with index $w_t$ is the $w_t$-th element of $v_t$, which is abbreviated as 

%\[v_t(i_{w_t})\].

\begin{equation}
%\[
v_t(w_t)
%\]
\label{eq:word_likelihood}
\end{equation}

\subsubsection{Loss Function}
We perform end-to-end training to learn all the aforementioned parameters together. % at the same time. %, using summation of loss across pairs in training data.
For each pair of a source sequence $Src$ and a target sequence $Tgt$ in the training set, where $Tgt=w_{t_1}...w_{t_n}$, by first encoding $Src$ through encoder, the loss of this pair is a summation of per-word cross-entropy loss between $v_i$ and the label which is a one-hot indicator vector of each word $w_{t_i}$.

\begin{figure*}
\centering
\resizebox{0.8\linewidth}{!}{
\begin{tikzpicture}

\node (x_t) {$\text{word}_t$};
\node[draw, align=center, inner sep=1] (embed) [above= 5mm of x_t] {Embedding\\ Layer};
\node[draw, align=center, inner sep=1] (unit_t1) [above= 5mm of embed] {LSTM \\ Layer 1};
\node[draw, align=center, inner sep=1] (unit_t2) [above= 5mm of unit_t1] {LSTM \\ Layer 2};
\node (dots) [above= 5mm of unit_t2]{$\cdots$};
\node[draw, align=center, inner sep=1] (unit_tl) [above= 5mm of dots] {LSTM \\ Layer $l$};
\node[align=center, inner sep=1] (unit_tprev1) [left= 10mm of unit_t1.north] {$\vec h_{t-1}^{(1)}$};
\node[align=center, inner sep=1] (unit_tnext1) [right= 10mm of unit_t1.north] {$\vec h_{t}^{(1)}$};
\node[align=center, inner sep=1] (unit_tprev2) [left= 10mm of unit_t2.north] {$\vec h_{t-1}^{(2)}$};
\node[align=center, inner sep=1] (unit_tnext2) [right= 10mm of unit_t2.north] {$\vec h_{t}^{(2)}$};
\node[align=center, inner sep=1] (unit_tprevl) [left= 10mm of unit_tl.north] {$\vec h_{t-1}^{(l)}$};
\node[align=center, inner sep=1] (unit_tnextl) [right= 10mm of unit_tl.north] {$\vec h_{t}^{(l)}$};

%first word
\node (x_t-1) [left= 29mm of x_t]  {$\text{word}_{0}$};
\node[draw, align=center, inner sep=1] (embed-1) [above= 5mm of x_t-1] {Embedding\\ Layer};
\node[draw, align=center, inner sep=1] (unit_t1-1) [above= 5mm of embed-1] {LSTM \\ Layer 1};
\node[draw, align=center, inner sep=1] (unit_t2-1) [above= 5mm of unit_t1-1] {LSTM \\ Layer 2};
\node (dots-1) [above= 5mm of unit_t2-1]{$\cdots$};
\node[draw, align=center, inner sep=1] (unit_tl-1) [above= 5mm of dots-1] {LSTM \\ Layer $l$};
\path[->] (x_t-1) edge (embed-1);
\path[->] (embed-1) edge (unit_t1-1);
\path[->] (unit_t1-1) edge (unit_t2-1);
\path[->] (unit_t2-1) edge (dots-1);
\path[->] (dots-1) edge (unit_tl-1);

%reverse direction
\node[align=center, inner sep=1] (bunit_tprev1) [left= 10mm of unit_t1.south] {$\cev h_{t-1}^{(1)}$};
\node[align=center, inner sep=1] (bunit_tnext1) [right= 10mm of unit_t1.south] {$\cev h_{t}^{(1)}$};
\node[align=center, inner sep=1] (bunit_tprev2) [left= 10mm of unit_t2.south] {$\cev h_{t-1}^{(2)}$};
\node[align=center, inner sep=1] (bunit_tnext2) [right= 10mm of unit_t2.south] {$\cev h_{t}^{(2)}$};
\node[align=center, inner sep=1] (bunit_tprevl) [left= 10mm of unit_tl.south] {$\cev h_{t-1}^{(l)}$};
\node[align=center, inner sep=1] (bunit_tnextl) [right= 10mm of unit_tl.south] {$\cev h_{t}^{(l)}$};

\node (ldots1) [left= 5mm of unit_tprev1]{$\cdots$};
\node (ldots2) [left= 5mm of unit_tprev2]{$\cdots$};
\node (ldotsl) [left= 5mm of unit_tprevl]{$\cdots$};
\node (rdots1) [left= 5mm of bunit_tprev1]{$\cdots$};
\node (rdots2) [left= 5mm of bunit_tprev2]{$\cdots$};
\node (rdotsl) [left= 5mm of bunit_tprevl]{$\cdots$};

\node (x_t+1) [right= 15mm of x_t]  {$\text{word}_{t+1}$};
\node[draw, align=center, inner sep=1] (embed+1) [above= 5mm of x_t+1] {Embedding\\ Layer};
\node[draw, align=center, inner sep=1] (unit_t1+1) [above= 5mm of embed+1] {LSTM \\ Layer 1};
\node[draw, align=center, inner sep=1] (unit_t2+1) [above= 5mm of unit_t1+1] {LSTM \\ Layer 2};
\node (dots+1) [above= 5mm of unit_t2+1]{$\cdots$};
\node[draw, align=center, inner sep=1] (unit_tl+1) [above= 5mm of dots+1] {LSTM \\ Layer $l$};

\path[->] (x_t) edge (embed);
\path[->] (embed) edge (unit_t1);
\path[->] (unit_t1) edge (unit_t2);
\path[->] (unit_t2) edge (dots);
\path[->] (unit_tprev1) edge (unit_t1.148);
\path[->] (unit_t1.32) edge (unit_tnext1);
\path[->] (unit_tprev2) edge (unit_t2.148);
\path[->] (unit_t2.32) edge (unit_tnext2);
\path[->] (dots) edge (unit_tl);
\path[->] (unit_tprevl) edge (unit_tl.148);
\path[->] (unit_tl.32) edge (unit_tnextl);

\path[->] (x_t+1) edge (embed+1);
\path[->] (embed+1) edge (unit_t1+1);
\path[->] (unit_t1+1) edge (unit_t2+1);
\path[->] (unit_t2+1) edge (dots+1);
\path[->] (unit_tnext1) edge (unit_t1+1.148);
\path[->] (unit_tnext2) edge (unit_t2+1.148);
\path[->] (unit_tnextl) edge (unit_tl+1.148);
\path[->] (dots+1) edge (unit_tl+1);

\path[<-] (bunit_tprev1) edge (unit_t1.212);
\path[<-] (bunit_tprev2) edge (unit_t2.212);
\path[<-] (bunit_tprevl) edge (unit_tl.212);
\path[<-] (bunit_tnext1) edge (unit_t1+1.212);
\path[<-] (bunit_tnext2) edge (unit_t2+1.212);
\path[<-] (bunit_tnextl) edge (unit_tl+1.212);
\path[->] (bunit_tnext1) edge (unit_t1.328);
\path[->] (bunit_tnext2) edge (unit_t2.328);
\path[->] (bunit_tnextl) edge (unit_tl.328);

\path[->] (ldots1) edge (unit_tprev1);
\path[->] (ldots2) edge (unit_tprev2);
\path[->] (ldotsl) edge (unit_tprevl);
\path[<-] (rdots1) edge (bunit_tprev1);
\path[<-] (rdots2) edge (bunit_tprev2);
\path[<-] (rdotsl) edge (bunit_tprevl);

\path[<-] (ldots1) edge (unit_t1-1.32);
\path[<-] (ldots2) edge (unit_t2-1.32);
\path[<-] (ldotsl) edge (unit_tl-1.32);
\path[->] (rdots1) edge (unit_t1-1.328);
\path[->] (rdots2) edge (unit_t2-1.328);
\path[->] (rdotsl) edge (unit_tl-1.328);

\node[align=center, inner sep=1] (unit_t1+2) [right= 10mm of unit_t1+1.north] {$\vec h_{t+1}^{(1)}$};
\node[align=center, inner sep=1] (unit_t2+2) [right= 10mm of unit_t2+1.north] {$\vec h_{t+1}^{(2)}$};
\node[align=center, inner sep=1] (unit_tl+2) [right= 10mm of unit_tl+1.north] {$\vec h_{t+1}^{(l)}$};
\node[align=center, inner sep=1] (bunit_t1+2) [right= 10mm of unit_t1+1.south] {$\cev h_{t+1}^{(1)}$};
\node[align=center, inner sep=1] (bunit_t2+2) [right= 10mm of unit_t2+1.south] {$\cev h_{t+1}^{(2)}$};
\node[align=center, inner sep=1] (bunit_tl+2) [right= 10mm of unit_tl+1.south] {$\cev h_{t+1}^{(l)}$};

\path[->] (unit_t1+1.32) edge (unit_t1+2);
\path[->] (unit_t2+1.32) edge (unit_t2+2);
\path[->] (unit_tl+1.32) edge (unit_tl+2);
\path[<-] (unit_t1+1.328) edge (bunit_t1+2);
\path[<-] (unit_t2+1.328) edge (bunit_t2+2);
\path[<-] (unit_tl+1.328) edge (bunit_tl+2);

\node[align=center, inner sep=1] (dots1+2) [right= 5mm of unit_t1+2] {$\cdots$};
\node[align=center, inner sep=1] (dots2+2) [right= 5mm of unit_t2+2] {$\cdots$};
\node[align=center, inner sep=1] (dotsl+2) [right= 5mm of unit_tl+2] {$\cdots$};
\node[align=center, inner sep=1] (bdots1+2) [right= 5mm of bunit_t1+2] {$\cdots$};
\node[align=center, inner sep=1] (bdots2+2) [right= 5mm of bunit_t2+2] {$\cdots$};
\node[align=center, inner sep=1] (bdotsl+2) [right= 5mm of bunit_tl+2] {$\cdots$};

\path[->] (unit_t1+2) edge (dots1+2);
\path[->] (unit_t2+2) edge (dots2+2);
\path[->] (unit_tl+2) edge (dotsl+2);
\path[<-] (bunit_t1+2) edge (bdots1+2);
\path[<-] (bunit_t2+2) edge (bdots2+2);
\path[<-] (bunit_tl+2) edge (bdotsl+2);

%decoder

\node (y_t)[right=45mm of  x_t+1]{$\langle\text{EOS}\rangle$};
\node[draw, align=center, inner sep=1] (dembed) [above= 5mm of y_t] {Embedding\\ Layer};
\node[draw, align=center, inner sep=1] (dunit_t1) [above= 5mm of dembed] {LSTM \\ Layer 1};
\node[draw, align=center, inner sep=1] (dunit_t2) [above= 5mm of dunit_t1] {LSTM \\ Layer 2};
\node (ddots) [above= 5mm of dunit_t2]{$\cdots$};
\node[draw, align=center, inner sep=1] (dunit_tl) [above= 5mm of ddots] {LSTM \\ Layer $l$};
\node[draw, align=center, inner sep=1] (datten) [above= 5mm of dunit_tl] {Attention \\Layer};
\node[draw, align=center, inner sep=1] (dproj) [above= 5mm of datten] {Projection \\Layer};

\node[align=center, inner sep=1] (dunit_tprev1) [left= 5mm of dunit_t1] {${ 
    \begin{bmatrix}
		   \vec h_{m}^{(1)}\\
           \cev h_{-1}^{(1)}
         \end{bmatrix}}$};
\node[align=center, inner sep=1] (dunit_tnext1) [right= 5mm of dunit_t1] {$h_{1}^{(1)}$};
\node[align=center, inner sep=1] (dunit_tprev2) [left= 5mm of dunit_t2] {${ 
    \begin{bmatrix}
		   \vec h_{m}^{(2)}\\
           \cev h_{-1}^{(2)}
         \end{bmatrix}}$};
\node[align=center, inner sep=1] (dunit_tnext2) [right= 5mm of dunit_t2] {$h_{1}^{(2)}$};
\node[align=center, inner sep=1] (dunit_tprevl) [left= 5mm of dunit_tl] {${ 
    \begin{bmatrix}
		   \vec h_{m}^{(l)}\\
           \cev h_{-1}^{(l)}
         \end{bmatrix}}$};
\node[align=center, inner sep=1] (dunit_tnextl) [right= 5mm of dunit_tl] {$h_{1}^{(l)}$};

\node (y_t+1) [right= 15mm of y_t]  {$\text{word}_{1}$};
\node[draw, align=center, inner sep=1] (dembed+1) [above= 5mm of y_t+1] {Embedding\\ Layer};
\node[draw, align=center, inner sep=1] (dunit_t1+1) [above= 5mm of dembed+1] {LSTM \\ Layer 1};
\node[draw, align=center, inner sep=1] (dunit_t2+1) [above= 5mm of dunit_t1+1] {LSTM \\ Layer 2};
\node (ddots+1) [above= 5mm of dunit_t2+1]{$\cdots$};
\node[draw, align=center, inner sep=1] (dunit_tl+1) [above= 5mm of ddots+1] {LSTM \\ Layer $l$};
\node[draw, align=center, inner sep=1] (datten+1) [above= 5mm of dunit_tl+1] {Attention \\Layer};
\node[draw, align=center, inner sep=1] (dproj+1) [above= 5mm of datten+1] {Projection \\Layer};

\iffalse
\path[->] (dots1+2) edge (dunit_tprev1);
\path[->] (dots2+2) edge (dunit_tprev2);
\path[->] (dotsl+2) edge (dunit_tprevl);
\fi

\path[->] (y_t) edge (dembed);
\path[->] (dembed) edge (dunit_t1);
\path[->] (dunit_t1) edge (dunit_t2);
\path[->] (dunit_t2) edge (ddots);
\path[->] (dunit_tprev1) edge (dunit_t1);
\path[->] (dunit_t1) edge (dunit_tnext1);
\path[->] (dunit_tprev2) edge (dunit_t2);
\path[->] (dunit_t2) edge (dunit_tnext2);
\path[->] (ddots) edge (dunit_tl);
\path[->] (dunit_tprevl) edge (dunit_tl);
\path[->] (dunit_tl) edge (dunit_tnextl);
\path[->] (dunit_tl) edge (datten);
\path[->] (datten) edge (dproj);

\path[->] (y_t+1) edge (dembed+1);
\path[->] (dembed+1) edge (dunit_t1+1);
\path[->] (dunit_t1+1) edge (dunit_t2+1);
\path[->] (dunit_t2+1) edge (ddots+1);
\path[->] (dunit_tnext1) edge (dunit_t1+1);
\path[->] (dunit_tnext2) edge (dunit_t2+1);
\path[->] (dunit_tnextl) edge (dunit_tl+1);
\path[->] (ddots+1) edge (dunit_tl+1);
\path[->] (dunit_tl+1) edge (datten+1);
\path[->] (datten+1) edge (dproj+1);

\node[align=center, inner sep=1] (dunit_t1+2) [right= 5mm of dunit_t1+1] {$h_{2}^{(1)}$};
\node[align=center, inner sep=1] (dunit_t2+2) [right= 5mm of dunit_t2+1] {$h_{2}^{(2)}$};
\node[align=center, inner sep=1] (dunit_tl+2) [right= 5mm of dunit_tl+1] {$h_{2}^{(l)}$};
\node[align=center, inner sep=1] (dsoftmax) [above= 5mm of dproj] {$v_1$};
\node[align=center, inner sep=1] (dsoftmax+1) [above= 5mm of dproj+1] {$v_2$};

\path[->] (dunit_t1+1) edge (dunit_t1+2);
\path[->] (dunit_t2+1) edge (dunit_t2+2);
\path[->] (dunit_tl+1) edge (dunit_tl+2);

\path[->] (unit_tl+1.north)[bend left=15] edge (datten);
\path[->] (unit_tl.north)[bend left=10] edge (datten);
\path[->] (unit_tl+1.north)[bend left=25] edge (datten+1);
\path[->] (unit_tl.north)[bend left=20] edge (datten+1);
\path[->] (unit_tl-1.north)[bend left=10] edge (datten);
\path[->] (unit_tl-1.north)[bend left=20] edge (datten+1);
\path[->] (dproj) edge (dsoftmax);
\path[->] (dproj+1) edge (dsoftmax+1);

\node (rdots1) [right= 5mm of dunit_t1+2]{$\cdots$};
\node (rdots2) [right= 5mm of dunit_t2+2]{$\cdots$};
\node (rdotsl) [right= 5mm of dunit_tl+2]{$\cdots$};
\path[->] ( dunit_t1+2) edge (rdots1);
\path[->] ( dunit_t2+2) edge (rdots2);
\path[->] ( dunit_tl+2) edge (rdotsl);

\draw [decorate,decoration={brace,mirror,amplitude=10pt},xshift=-4pt,yshift=0pt] (-4,-0.2) -- (5,-0.2) node [black,midway,yshift=-15pt] {Encoder};

\draw [decorate,decoration={brace,mirror,amplitude=10pt},xshift=-4pt,yshift=0pt] (7,-0.2) -- (14,-0.2) node [black,midway,yshift=-15pt] {Decoder};

\end{tikzpicture}
}
\caption{Bi-directional Multilayer LSTM Encoder + LSTM Attention Decoder}
\label{fig:seq2seq}

\end{figure*}
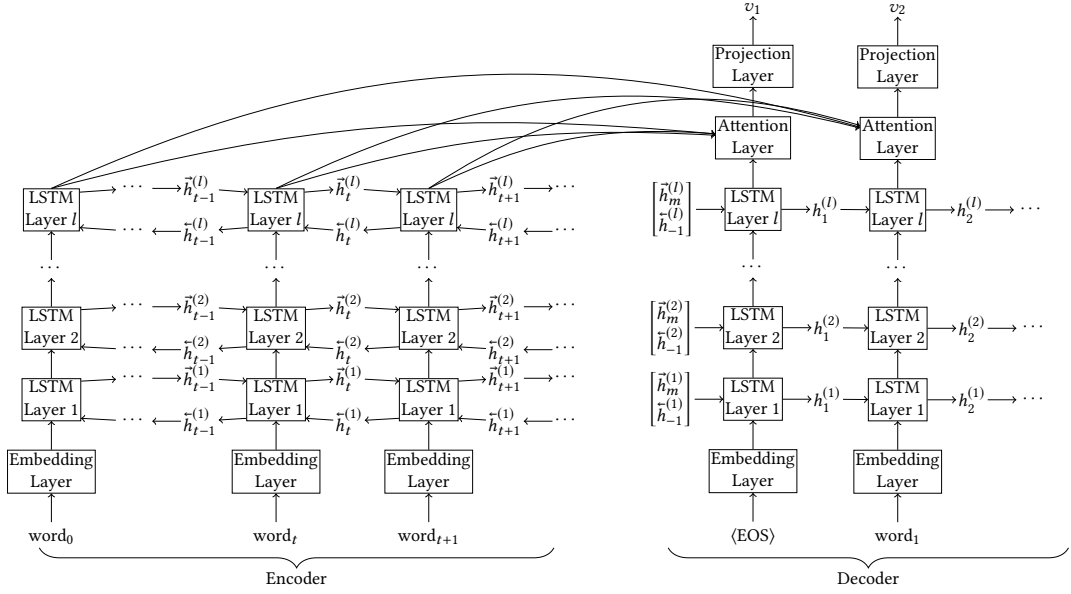

\subsection{Attention Mechanisms}\label{sec:atten}
Attention mechanism is a powerful add-on to recurrent neural networks that is intended to combat the long-term dependency issue. Even LSTM and GRU networks, which are designed to have long term dependencies, are prone to missing information that occurred long time ago. %To target this issue, in \cite{luong-pham-manning:2015:EMNLP} the authors discussed the attention mechanisms, with an application in neural machine translation. They demonstrated an improvement of the BLEU score by 5.0 for English-French translation. 
The intuition behind attention mechanism is that, at each step of the sequence decoding process, we force the network to look back again at the source sequence to pick up the most relevant hidden vectors, and augment the current hidden vector with this extra piece of information. 

To introduce the attention mechanism, we define a few notations for convenience purposes. Let the top hidden vectors for the source sequence be $s_1,\cdots,s_m$, and the top hidden vector of the current target word be $h_i$. We summarize four variants that we intend to experiment \cite{luong-pham-manning:2015:EMNLP}\cite{NIPS2013_5028}.
%In \cite{luong-pham-manning:2015:EMNLP}\cite{NIPS2013_5028}, three attention mechanisms were proposed, which were named {\it dot, general} and {\it concat} by the authors. 
These methods are different ways of averaging source hidden vectors $s_1,\cdots,s_m$, where the important words are supposed to have a larger weight, hence it gets the name of ``attention''. For the four mechanisms, the weight vector $a_{i\cdot}$ for target word $i$ can be calculated respectively firstly:

\begin{enumerate}
\item  {\it (dot)} $\tilde{a}_{ij}=s_j^T h_i$
\item  {\it (general)} $\tilde{a}_{ij}={s_j}^T W_g h_i$
\item  {\it (concat)} $\tilde{a}_{ij}=W_{cc} [s_j; h_i]$
\item 
{\it (tensor)} $\tilde{a}_{ij} = U({s_j}^T W h_i+ V [s_j; h_i] + b)$
\end{enumerate}
Then with $a_{i,\cdot} = \text{softmax}(\tilde{a}_{i\cdot})$, the attention vector is obtained by 
$g_i=\sum_{j=1}^m a_{ij} s_j$, 
which will be combined with $h_i$ and fed into the projection layer. Accordingly we learn parameters $W_g\in\mathbb R^{d_h \times d_h}$, $W_{cc}\in\mathbb R^{1x2d_h}$,  tensor $W\in\mathbb R^{d_h \times k \times d_h}$ $V\in\mathbb R^{k \times 2d_h}$, $b\in\mathbb R^{k}$,
and $U\in\mathbb R^{1 \times k}$.

The four different attention mechanisms aim at different purposes. While {\it dot} and {\it general} aim at discovering the similarities between source and target, the last one focuses more on the non-linearity interaction between words, as pointed out in recursive neural network literature \cite{socher2013recursive}. More comparisons and analysis will be discussed in sections \ref{application}.
%where $W\in\mathbb R^{d\times d\times d}$ is a tensor that reflects the non-linear operations between source and target hidden vectors, and $b\in\mathbb{R}^d$ is the bias vector. %We 

%In addition to the aforementioned three tensor networks, we propose a dynamic tensor attention mechanism, which is inspired by the recursive tensor networks proposed in \cite{socher2013recursive}. In natural languages, the effect of word operations are often non-linear. For example, ``very good'' enhances the meaning of ``good'' where ``very'' acts like a multiplicative operator, and in ``not good'', the operation of ``not'' should behave like a negation operator. These non-linearities can be better captured by tensor networks. Formally, the attention scheme follows 
%{\color{red}
%\[ U({s_j^{(l)}}^T W h_i^{(l)}+b)\]
%}
%where $W\in\mathbb R^{d\times d\times d}$ is a tensor that reflects the non-linear operations between source and target hidden vectors, and $b\in\mathbb{R}^d$ is the bias vector. %We discovered that attention indeed reduces cross validation loss during training, as compared with plain LSTM models. Moreover, \ref{fig:training_loss} shows that dynamic tensor network has the least loss among different attention schemes.
%\begin{figure}
%\includegraphics[width=\linewidth]{comp_5epochs.png}
%\caption{Training loss for different attentions}
%\label{fig:training_loss}
%\end{figure}

\subsection{Likelihood Estimation}\label{likelihood}

%\begin{enumerate}
%\item Sequence to Sequence Likelihood Estimation\\
The above seq2seq model is capable of giving an estimate of the likelihood of a target sequence, $Tgt=w_{t_1}...w_{t_n}$, given a source sequence $Src$. First, notice the chain rule for conditional probability, we have
\[Pr(Tgt|Src)=Pr(w_{t_1}|Src)\times\cdots\times Pr(w_{t_n}|w_{t_{n-1}},..,w_{t_1},Src)\]
For the $i$-th word in the target sequence, $w_{t_i}$, the conditional distribution $Pr(w_{t_i}|w_{t_{i-1}},..,w_{t_1},Src)$ is estimated by the seq2seq model as in Equation (\ref{eq:word_likelihood}), 
\[\widehat{Pr}(w_{t_i}|w_{t_{i-1}},..,w_{t_1},Src)=v_{i}(w_{t_i})\]
Combine the chain rule step with the seq2seq estimator, we obtain the estimated sequence likelihood, which is 

\begin{equation}
Pr(Tgt|Src)=\prod_{i=1}^n v_{i}(w_{t_i})
\label{eq:likelihood}
\end{equation}

%Na\" ive Bayes Prior Update\\
% Denote $\pi$ as a prior distribution on the set of all source sequences. Now, suppose $k$ target sequences $Tgt_1^k$, are revealed, the posterior distribution on the set of source sequences should change accordingly, reflecting the fact that more information is provided by the newly revealed target sequences. To ground us to the realm of  recommendation systems, we can imagine the source sequence $Src$ as an item to be recommended, and  target sequence $Tgt_i$ as the user input at the $ i$-th round.  
%  Applying the Bayes rule, we have
% \[Pr(Src|Tgt_1^k) = \frac{\pi(Src)Pr(Tgt_1^k|Src)}{\sum_{Src} \pi(Src)Pr(Tgt_1^k|Src)}\]
% Under a na\"ive Bayes framework, we assume conditional independence, 
% $Pr(Tgt_1^k|Src)=\prod_{i=1}^k Pr(Tgt_i|Src)$.
% Combining the above two expressions, the prior update rule becomes 
% \begin{equation}
% Pr(Src|Tgt_1^k) = \frac{\pi(Src)\prod_{i=1}^k Pr(Tgt_i|Src)}{\sum_{Src} \pi(Src)\prod_{i=1}^k Pr(Tgt_i|Src)}
% \label{eq:posterior}
% \end{equation}
% where we notice that each likelihood term, $Pr(Tgt_i|Src)$, is given by the seq2seq likelihood estimator in the previous step.
%\end{enumerate}

\section{Information-Directed Adaptive Sequence Sampling}
In the above discussions, we focused on deep learning models and attention mechanisms. However, its ability was mostly investigated in traditional, non-adaptive and one-shot inference scenarios. By non-adaptive and one-shot, we mean that the data are given to the algorithm {\em as is}, with no control over the data collecting process whatsoever. Most machine learning algorithms are designed to cope with this scenario, but the rise of new interactive channels like chatbot, virtual agents or interactive webpages demand further. An agent, like a chatbot, has to have the adaptivity of talking and raising clarifying questions to a user to reach his or her goal.
So our framework is created to address this. By {\em adaptive}, it means that it is able to dynamically sample the next user input depending on the current estimates, hence will be more directional and less ad-hoc. In other words, it should interpret user intent and knowingly guide the user to achieve the goal in the most efficient way. Next we will explain how \SYS integrates the seq2seq model to do the estimation, identify the next sampling direction, and make recommendations when the agent is confident.

\subsection{Recommending an Item}
Consider a scenario where we would like to make recommendations. Denote $\pi$ as a prior distribution on the set of all possible items. In this setting, each $Item$ can be represented as a sequence, for example the title of an ad. Now, suppose $k$ input sequences $Input_1^k$ from a user are revealed, e.g. from $k$ rounds of interactions, the posterior distribution on the set of items should change accordingly, reflecting the fact that more information is provided by the user.  Applying the Bayes rule, we have
\[Pr(Item|Input_1^k) = \frac{\pi(Item)Pr(Input_1^k|Item)}{\sum_{Item} \pi(Item)Pr(Input_1^k|Item)}\]
Under a na\"ive Bayes framework, we assume conditional independence, 
$Pr(Input_1^k|Item)=\prod_{i=1}^k Pr(Input_i|Item)$.
Combining the two expressions, the update rule becomes 
\begin{equation}
Pr(Item|Input_1^k) = \frac{\pi(Item)\prod_{i=1}^k Pr(Input_i|Item)}{\sum_{Item} \pi(Item)\prod_{i=1}^k Pr(Input_i|Item)}
\label{eq:posterior}
\end{equation}
where we notice that each likelihood term, $Pr(Input_i|Item)$, is given by the seq2seq likelihood estimator in Equation (\ref{eq:likelihood}).

\subsection{Entropy as a Measure of Confidence}
\subsubsection{Definition and Discussion of Intuition}
Entropy is a functional of a probability distribution, which measures how unpredictable the distribution is. We use it to determine the confidence of an agent, or how vaguely the situation is to the agent. It originated from information theory which quantifies the compressibility of a IID random source sequence \cite{cover2012elements}, but has since been widely applied to other fields, including computer vision and speech recognition. For example, the maximum entropy principle, first proposed by Hoch and Skilling \cite{skilling1984maximum} \cite{hoch1996maximum}, has shown extreme success in image reconstruction and de-blurring. The max entropy principle has found applications in speech recognition, where an example is a speech recognition system \cite{peters2006speech} built by Peters et. al. In NLP, language models, as in \cite{khudanpur1999maximum}, are sometimes built around this idea as well. We also point out that in NLP, the notion of {\it perplexity}, a standard metric used to compare statistical language models and machine translation such as in \cite{sutskever2014sequence}, can be viewed as the exponent of the entropy. Below we give a formal definition of the entropy functional. Notice in the following definition of conditional entropy, it is not averaged across the random variable it conditions on, hence is itself a random variable.

\begin{definition}[Entropy, Conditional Entropy]
Given a pair of discrete random variables $(X,Y)$, where $X$ takes values from a alphabet $\mathcal X$ and $Y$ takes value in $\mathcal Y$. Denote their joint distribution as $p_{X,Y}(x,y)$ and marginals $p_X(x),p_Y(y)$,
\begin{enumerate}
\item
The entropy of $X$ is defined as
\[H(X)=-\sum_{x\in\mathcal X} p_X(x)\log p_X(x)\]
\item
The conditional entropy of $X$ given $Y=y$ is
\[H(X|Y=y)=-\sum_{x\in\mathcal X} p_{X|Y}(x|y)\log p_{X|Y}(x|y)\]
Finally, we use $H(X|Y)=\sum p_Y(y)H(X|Y=y)$ to denote the {\em expected} conditional entropy of $X$ given $Y$.
\end{enumerate}
\end{definition}
In general, a large entropy is an indication of the distribution being more widespread. For example, when entropy is maximized, $X$ has a uniform distribution. On the contrary, when entropy is small, the distribution is more concentrated. $H(X)=0$ effectively means the distribution is deterministic. 

\subsubsection{Uncertainty of Sequence Posterior Estimation}
The conditional entropy can serve as an uncertainty measure of the estimated sequence posterior distribution. %, thanks to its indicative property of the inherit unpredictability. 
Remember in Equation (\ref{eq:posterior}), we discussed the posterior update procedure when $k$ user inputs, $Input_1^k$ are observed. We define the {\it posterior uncertainty} as the entropy of this conditional distribution, $H(Item|Input_1^k)$. A large posterior uncertainty means the estimation is vague, hence more observations are needed before a decision can be made; on the other hand, a posterior uncertainty close to 0 is an indication of the estimation has pretty much converged to its argmax, under which case a sure recommendation is ready to be made. Next we explain how to sample more observations or determine the best question to ask if it's uncertain.

\subsection{Information-directed Sampling: Principle of Maximizing Expected Information Gain}
\subsubsection{Mutual Information}
Originated from information theory, the mutual information quantifies how much information can be reliably communicated through a channel. It is a functional on a pair of random variables $(X,Y)$, which is a measure of how much knowledge one can gain of $X$ when $Y$ is revealed. It is defined as the difference between the entropy of $X$ and the conditional entropy of $X$ given $Y$.
\begin{definition}[Mutual Information]
The mutual information between $(X,Y)$ is
\[I(X;Y)=H(X)-H(X|Y)\]
Similarly, conditioning on a sequence of random variables $Z_1^k=z_1^k$, the mutual information between $(X,Y)$ is
\[I(X;Y|Z_1^i=z_1^i)=H(X|Z_1^i=z_1^i)-H(X|Y,Z_1^i=z_1^i)\]
\end{definition}

\subsubsection{Information-Directed Sampling Algorithm}
Now suppose the agent is able to proactively interact with the user, being able to ask the user with questions and expects answers from the user. To start with, assume there is a set of questions, $\mathcal Q=\{Qst_1,\cdots, Qst_q\}$. Following the maximizing information gain principle, we propose Algorithm \ref{alg:algo1}.

\begin{algorithm}
\caption{Information-directed Sequence Sampling}
\label{alg:algo1}
\begin{algorithmic}[1]
\For{ $n=1,2, \cdots$}
 \State The $n$-th sequence $Input_n$ is collected from the user.
 \State Estimate the likelihood, $Pr(Input_n|Item)$, using the seq2seq likelihood estimator.
 \State Update the posterior distribution
 \[Pr(Item|Input_1^n) = \frac{\pi(Item)\prod_{i=1}^n Pr(Input_i|Item)}{\sum_{Item} \pi(Item)\prod_{i=1}^k Pr(Input_i|Item)}\]
 \State Calculate the conditional entropy $H(Item|Input_1^n)$
 \If{$H(Item|Input_1^n)<T$}
  \State Return $\argmax Pr(Item|Input_1^n)$,  the most likely item.
 \Else
  \State Choose $Qst$ that maximizes $I(Qst;Item|Input_1^n)$
  \State Propose $Qst$ to user; wait for user feedback $Input_{n+1}$
 \EndIf
 \EndFor
\end{algorithmic}
\end{algorithm}

We would like to point out that the Algorithm \ref{alg:algo1} which maximizes the expected information gain at each step, is effectively a greedy uncertainty-reduction algorithm. This observation is stated in the lemma below. 
\begin{lemma}
The information-gain maximizing $Qst$ proposed at step $n$ is also a uncertainty minimizer at step $n$.
\end{lemma}
\begin{proof}
Note that
\[I(Qst;Item|Input_1^n)=H(Item|Input_1^n)-H(Item|Qst,Input_1^n),\]
Note $H(Item|Input_1^n)$ does not depend on $Qst$, as a result, the maximizer of $I(Qst;Item|Input_1^n)$ is immediately a minimizer of $H(Item|Qst,Input_1^n)$ and vice versa.
\end{proof}
A discussion on the application of a Chatbot and question formulation procedure will be discussed in section \ref{sec:chatbot}. 

%% file: application.tex
\section{Applications}\label{application}
\SYS is versatile in the sense that it has a wide range of applicability, covering inference, ranking, adaptive sampling and decision making. We have built three applications of \SYS, illustrated in figure \ref{three}. In this section, we will elaborate on the details including the design, training, implementation and evaluations of these applications. Note that we build the applications on top of a commercial ``product ads'' search engine, which recommends products if a user inputs a query with product intent.

\noindent
\begin{figure}
\centering
%\hspace*{-1cm}
\resizebox{0.8\columnwidth}{!}{%
\begin{tikzpicture}[font=\sffamily,>=stealth',thick,
commentl/.style={text width=3cm, align=right},
commentr/.style={commentl, align=left},]
\node[] (init) {\LARGE User};
\node[right=2cm of init] (recv) {\LARGE Agent};

\draw[->] ([yshift=-0.7cm]init.south) coordinate (fin1o) -- ([yshift=-.7cm]fin1o-|recv) coordinate (fin1e) node[rectangle, fill=green, pos = .5, align=center, above, sloped, draw] {Ingress:\\Query Rewriter}
node[pos = .5, align=center, below, sloped] {{\itshape``tablet TV connector"}};

\draw[->] ([yshift=-.3cm]fin1e) coordinate (ack1o) -- ([yshift=-2.3cm]ack1o-|recv) coordinate (ack1e) 
node[rectangle, fill=green, pos = .4, align=center, right, draw] {Processing:\\Info. Directed \\Decision Maker}
node[pos = .4, align=center, left] {``Recommend or\\ ask more questions?};

\draw[->] (ack1e-|recv) coordinate (fin2o) -- ([yshift=-.7cm]fin2o-|init) coordinate (fin2e) node[rectangle, fill=green, pos = .5, align=center, above, sloped, draw] {Egress:\\ Relevance Scoring};

\draw[->] ([yshift=-.3cm]fin2e) coordinate (ack2o) -- ([yshift=-.7cm]ack2o-|recv) coordinate (ack2e) node[pos=.5, above, sloped] {Query (If needed)};

\draw[thick, shorten >=-1cm] (init) -- (init|-ack2e);
\draw[thick, shorten >=-1cm] (recv) -- (recv|-ack2e);

\draw[dotted] (recv.285)--([yshift=2mm]recv.285|-fin1e) coordinate[pos=.5] (aux1);

\draw[dotted] (init.255)--([yshift=2mm]init.255|-fin1o);

\draw[dotted] ([yshift=1mm]init.255|-fin2e) --([yshift=-5mm]init.255|-ack2e) coordinate (aux2);

\node[commentr, right =2mm of ack2e] {\textbf{...}};
\node[below left = 0mm and 2mm of init.south, commentl]{\textbf{INITIAL QUERY}\\[-1.5mm]{\itshape ``How do I connect tablet to TV?"}};
\node[below left = -1mm and 2mm of aux2-|init, commentl]{\textbf{...}};
\end{tikzpicture}
}
\caption{Three applications (in green) of \SYS}
\label{three}
\end{figure}
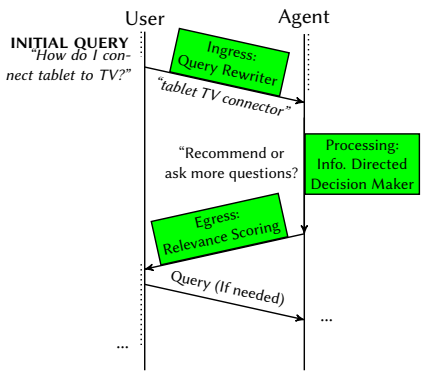

\subsection{Ingress: Query Rewriting}\label{sec:rewriting}
Query understanding and rewriting is a vital pre-processing step for modern recommendation and information retrieval systems. In the area of product ads recommendation, a user input query in the search engine should be able to 1) trigger an ad recommendation action and 2) return the relevant ads. If the query is in a standard form, like it's grammatically correct, and contains the right keywords, the backend information retrieval system will be able to return the recommendations, and the response time must be short to ensure a high quality of service.  In the first application, we apply \SYS for query rewriting. 
%The seq2seq model is used in a beam-search fashion; at each step, the most likely k words were fed to the decoder as the next input until a $\langle\text{EOS}\rangle$ token is met, thus allowing it to generate target sequences. $k$ is set to 1 in the evaluation below.

\subsubsection{What \& How: Question Understanding}
A pain point we identified about product ad recommendation in our search engine is that it does not process queries in question form well. These queries are often ambiguous, and the product is implicitly referred to, usually formulated in a relationship to other entities. As an example, a user might type in the search box a question like 
\[ How\ to\ connect\ my\ tablet\ to\ TV?\]
From a human point of view, this query clearly points to a product: micro HDMI cable. However, this posts a challenge to the information retrieval system, as no clear keywords related to the right product ad were present in the query. 

\subsubsection{Training and Data}
%\subsubsection{Data Description} 
\label{sec:rewrite_data}
We trained \SYS to generate standard queries from question-form queries. To serve this purpose, we used data collected from a ``related searches" feature on a commercial search engine. The related searches are a list of queries being recommended to a user when a specific query is typed in the search box, and many of them are standard queries. We picked user-input queries starting with ``what" and ``how", and regard a related search query as a positve training example if it was clicked by the user. The click behavior by a user confirms that the standard query is indeed relevant to the question the user has entered. By doing so, we were able to collect a dataset consisting of 12 million clicked (question-form query, standard query) pairs. To further focus on questions that will end up with product ad recommendation, we filter the dataset by keeping only the pairs where there was product ad recommendation for the standard query itself. A total of 782 thousand such training pairs were collected. We summarize the statistics of the training dataset in Table \ref{tb:rewrite_train}. 

\begin{table}
\centering
\begin{tabular}{ l | c c c c}
 & size & vocabulary & average length & clicks \\
 \hline
questions & 316K & 126K %53K
& 5.5 & 782K \\
queries & 481K & 870K %85K
& 2.8 & 782K
\end{tabular}
\caption{Statistics of the Rewrite Training Set}
\label{tb:rewrite_train}
\end{table}

\begin{table}
\centering
\begin{tabular}{l | l| l }
\# questions & \# queries & \# pairs  \\
\hline
34K & 42K & 45K 
\end{tabular}
\caption{Statistics of the Rewrite Test Set}
\label{tb:rewrite_test}
\end{table}

\subsubsection{Details of Model}
We used the model in section \ref{sec:seq2seq} with vocabulary size $|\mathcal V|=100k$ for both the encoder and decoder. Any word not in $\mathcal V$ is assigned with symbol $\langle$UNK$\rangle$. %
We chose the embedding dimension $d_{emb}=100$. We used 3-layer LSTMs with hidden vector size $d_h$=300 on the decoder side, and we implemented 4 different attention scenarios as in Section \ref{sec:atten}. The results for the four different attention mechanisms are compared. The model rewrites to a sequence of words as follows. 
At step $i$ at the decoder, the model picks the most likely word and use it as the input to the embedding layer at step $i+1$, until the max length is reached, or an $\langle$EOS$\rangle$ token is encountered. We used Theano \cite{2016arXiv160502688short} for model training on a Tesla K20 GPU, with cross entropy as the loss function, and Adadelta \cite{zeiler2012adadelta}, a variant of Adagrad \cite{duchi2011adaptive}, for gradient descent. We do end-to-end training to learn all the parameters described in Section \ref{sec:seq2seq}, with a total of 10 epochs.%, each epoch takes 2 hours. 

\subsubsection{Result and Evaluation}

In our experiments, we found \SYS's rewriting helped in two ways. First, while many original queries are product related, they did not trigger product ads, due to the form in which the queries are presented, or their implicitness. After being rewritten, they become more keyword-like and trigger product ads. 
Some of such examples are
\begingroup\makeatletter\def\f@size{7}\check@mathfonts
\begin{align*}
\text{How to connect my tablet to TV}&\rightarrow\text{ tablet tv connector}\\
\text{How to repair my broken iphone screen}&\rightarrow\text{ iphone screen replacement}\\
\text{How to charge my iphone}&\rightarrow\text{ iphone charger}\\
\text{How to protect my iphone screen}&\rightarrow\text{ iphone screen protector}
\end{align*}
\endgroup

Secondly, rewriting also helps in retrieving the correct ads, especially when implicit or complex relations are present in the query. To provide an explicit example, the query ``How to wire car radio" indicates, from the human understanding perspective, that the user has the radio already and is looking for wiring products. When submitted in the original form, ads on car radios are retrieved. After \SYS rewrites it to ``radio wiring", the correct ads (radio wiring harness) are retrieved. Another example is the query ``How to fix gps in car", where in its original form it triggers ad about mobile GPS, and after rewriting, the correct ads, GPS holders are returned. 

As a quantitative evaluation, we test on a different (question-form query, standard query) test dataset. The test set was collected using the same procedure described in Section \ref{sec:rewrite_data}, but sampled from log in a different time period. In addition, any pair appearing in the training dataset was removed from the test set. We summarize the statistics of the test dataset in Table \ref{tb:rewrite_test}.

\subsubsection{Quality of Rewrites}
For each pair in the test set, we generate rewrites using different \SYS model variations. We evaluate each rewrite against the standard query as baseline using BLEU score \cite{Papineni:2002:BMA:1073083.1073135}. BLEU (bilingual evaluation understudy) is an algorithm for evaluating the quality of text which has been machine-translated from one natural language to another. We borrow the same technique to evaluate query rewriting since the BLEU score is a standard evaluation for seq2seq model-based translation. We summarize the average BLEU score results in Table \ref{tb:BLEU}. We can clearly see that attention mechanisms consistently outperform the base model without attention. This observation is different from several recent attempt in applying seq2seq model for question-answering like \cite{DBLP:journals/corr/VinyalsL15}, which stated that attention mechanism is not helpful. 
If we reflect based on the examples shown above, this can be explained by the source-target sequence alignment characteristic of our question {\em rewriting} application, a property that machine translation shares but not question {\em answering}.
Among the attention mechanisms, general attention performs the best while concat attention is the worst. Strictly speaking concat attention does not do alignment directly using the hidden vectors of source and target words. On the other hand, dot and general attention does exactly that. Lastly, the tensor mechanism ranks the $2^{nd}$, only a bit worse than general attention. We suspect its structure is too complicated to learn when combined against recurrent neural networks.

\subsubsection{Quality of Ad Recommendation}
We evaluate the quality of ad recommendation using the rewrites as input to the ad system.
We first sampled 1000 question queries from the test pairs, and generate rewrites using different \SYS model variations. Then, we submit the rewrites to product ads search engine. The batch submission effort is usually called ``scraping'' in industry. We also scrape the system with the original question queries and the standard query respectively, in order to compare the ads coverage and quality. The results are presented in Table \ref{tb:scrape}. 
The ads coverage is defined by \% of questions having ads returned. 
For ads quality, we sample 3000 (original question,  ad) pairs for each version of rewrites and their returned ads. Each pair is labeled by a group of trained human judges according to the relevance between the query and the ad. Each label ranges in \{bad, fair, good, excellent\}, and we consider \{fair, good, excellent\} as positive. Note that along with the recommended ads, we submit the ``original question'' to judges. Judges are only comparing the original question to the returned ad, without knowing the ads are actually retrieved using rewrites. The ads quality is based on \% positive labeled ads in each 3000-pair set. 

In Table \ref{tb:scrape}, we see that only 21.0\% of the original questions triggered product ad recommendations. 21.1\% of the returned ads are of reasonable quality. If we scrape with the related search queries collected from the log, we see much higher ad coverage at 84.7\%. This is expected as we already filter the test set this way. More importantly, we see even better quality ads at 25.3\%. This confirms the validity of our rewrite data collection method. The clicked related search queries are indeed relevant to the questions so that, the ads returned using the rewrites are similarly relevant compared to scraping with the questions themselves. 

Among \SYS's rewrites, we see that using {\it general} attention achieves both the highest coverage and quality. We see a 3.5x increase in coverage and a $50\%$ increase in quality, relative to the original query. Even when compared with the unobserved ground-truth, i.e., the clicked related search queries, we see only a $10.7\%$ decrease in coverage but a $12.4\%$ increase in quality. 

\begin{table}[t]
\centering
\begin{tabular}{ l | c }
 Model & BLEU Score  \\
 \hline
\SYS rewrites without attention 
& 0.326  \\
\SYS rewrites with dot attention
& 0.349 \\  
\SYS rewrites with general attention
& {\bf 0.388} \\  
\SYS rewrites with concat attention
& 0.331 \\  
\SYS rewrites with tensor attention
& 0.364 \\  
\end{tabular}
\caption{BLEU scores between rewrites and standard queries}
\label{tb:BLEU}
\end{table}

\begin{table}[htb]
\centering
\resizebox{\columnwidth}{!}{%
\begin{tabular}{ l | c c c }
 Scrape Set & Ads Coverage %& Ad Density 
 & Ads Quality  \\
 \hline
Original questions & 21.0\% %& 1.41 
& 21.1\%  \\
Related search queries & 84.7\%  %& 39.91 
& 25.3\% \\
 \hline
\SYS rewrites without attention & 67.4\% %& 30.58 
& 26.7\%  \\
\SYS rewrites with dot attention & 64.2\% %&  
& 33.0\% \\  
\SYS rewrites with general attention & 74.0\% %&  
& {\bf 37.7\%} \\  
\SYS rewrites with concat attention & 64.1\% %&  
& 18.2\% \\  
\SYS rewrites with tensor attention & 64.1\% %& 28.57 
& 28.8\% 
\end{tabular}
}
\caption{Scraping results}
\label{tb:scrape}
\end{table}

\subsubsection{Discussion}
The results indicate that doing ``question''-rewriting with appropriate training data achieves improved recommendation quality and coverage. This staged approach can be seamlessly integrated into current  infrastructure. It does not require any change in the existing information retrieval system, as the rewritten query can be submitted either instead of or along with the original one.
In addition,  targeting only ``what'' and ``how'' questions is just the first step towards a general-purpose question-answering system. Readers can imagine that this application would be part of a large-scale, comprehensive system, where this application only focuses on product recommendation.  Lastly, one may argue that although a significant improvement is observed, the reported ads quality is still not high. This leads to the next section using seq2seq for scoring and keeping better candidates.

% We identify two major advantages of the \SYS application on query rewriting, which are its speed and compatibility.
% \begin{enumerate}
% \item
% Note that here we its the generative ability, which is essentially a forward pass. It can be handily accelerated using BLAS implementation as the operations are primarily vector-vector and vector-matrix multiplications. A typical generative run takes 4 ms. This is reasonable for many uses cases like chatbot conversation, but still considered not fast enough for small delay scenarios like search engine. Identifying the bottleneck as the projection layer in the decoder side, we found there are algorithms that can hopefully further speed up this procedure. For example, the space-partition data structure like k-d tree can be used to accelerate the argmax, or minimum-distance searching process.
% \item
% One other major advantage of query processing is that it can be seamlessly integrated into current infrastructure. It does not require any change in the existing information retrieval system, as the rewritten query can be submitted either instead of or along with the original one. The two advantages made it easy to deploy.
% \end{enumerate}

\subsection{Egress: Relevance Scoring}\label{sec:scoring}
\SYS's ability of estimating items' posterior distribution also makes it a good fit for quality control at the egress side. When a set of ads are returned from the information retrieval infrastructure, \SYS can serve as a relevance filter which shows only the most related ads to the user. %Algorithm \ref{alg:algo2} shows a scoring method with possibly more than one user input. As an isolated experiment, next we discuss a test case  with only one user input per ad.

% \begin{algorithm}
% \caption{Relevance Scoring}
% \label{alg:algo2}
% \begin{algorithmic}[1]
% \State Inputs $\text{Input}_1^k$ submitted with $\{\text{Item}_1,\cdots, \text{Item}_n\}$ returned from IR infrastructure
% \For{ $i=1,2, \cdots, n$}
%  \State Calculate $Pr(\text{Item}_i|\text{Input}_1^k)$ by Equation (\ref{eq:posterior})
% \EndFor 
% \State Sort $\{\text{Item}_1,\cdots, \text{Item}_n\}$ accordingly
% \State Display the top $m$ results.
% \end{algorithmic}
% \end{algorithm}

\subsubsection{Training and Data}
The \SYS scoring model was trained on our internal dataset, which consists of clicked (query, ad) pairs sampled from a commercial product ad search engine. The ads come from a product ad database, each is a sequence of words describing the corresponding product. The queries are user inputs in our search engine, and if the user clicked on an ad when searching with a query, we regard it as a positive (query, ad) pair. A total of 15 million clicks are sampled from a month long of click logs, which ends up with 6.4 million distinct user queries and 5.1 million distinct ads. We summarize the statistics of the training dataset in Table \ref{tb:data}.

\subsubsection{Details of Model}
We used the model in Section \ref{sec:seq2seq} and chose a vocabulary size $|\mathcal V_q|=60k$ on the query side and $|\mathcal V_d|=100k$ on the ad side. Any word that is not part of $\mathcal V$ is assigned with symbol $\langle$UNK$\rangle$. We chose the embedding dimension $d_e=150$ on both encoder and decoder sides, and hidden dimension $d_h$=300 on the decoder side.
Although we did notice an improvement in performance for deeper networks, in this experiment we trained a single layer LSTM model for a fair comparison noted below. We trained the model for 5 epochs.

%[better to talk about what are the parameters and we train them end to end]

\begin{table}
\centering
\begin{tabular}{ l | c c c c}
 & size & vocabulary & average length & clicks \\
 \hline
query & 6.4M & 68K & 4.1 & 15M \\
ads & 5.1M & 114K & 9.3 & 15M
\end{tabular}
\caption{Statistics of the Scoring Training Set}
\label{tb:data}
\end{table}

\begin{table}[t]
\begin{tabular}{l | l| l | l | l}
\# queries & \# ads & \# pairs & \# positive & \# negative \\
\hline
23K & 915K & 965K & 234K & 731K
\end{tabular}
\centering
\caption{Statistics of the Scoring Test Set}
\label{tb:testset}
\end{table}

\subsubsection{Evaluation}

As a quantitative evaluation, we test on a fully annotated test set. The test set contains around 966 thousand (query, ad) pairs where each pair is labeled by a group of trained human judges according to the relevance between the query and the ad. Each label ranges in \{bad, fair, good, excellent\}. The pairs are sampled from the early selection stage of a commercial ads search engine, where there are a significant amount of low quality selected ads to be pruned out in downstream processing. We use AUC (area-under-curve of the receiver operating characteristic plot) as the metric for evaluation, 
%by considering the bad label as the negative class and the rest labels as the positive class (fair, good and excellent). This results in a test set consisting of 369 thousand positive and 597 thousand negative pairs. 
by considering the good and excellent labels as the positive class and the rest labels as the negative class. This results in a test set consisting of 234 thousand positive and 731 thousand negative pairs. 
We briefly summarize the statistics of the testset in Table \ref{tb:testset}. We use a uniform prior $\pi$ on the set of ads, and hence $Pr(ad|query) \propto Pr(query|ad)$ as in Equation (\ref{eq:likelihood}). As a result, $Pr(query|ad)$ serves as the relevance score for a pair $(query, ad)$. The AUC is then computed according to the scores for all the pairs in the test set.

We compare against existing relevance scoring baselines, including the popular CDSSM \cite{shen2014learning} and DeepIntent \cite{zhaideepintent}, both of which are deep-learning based and have shown very satisfactory results in production. Given a (query, ad) pair, they encode the query and the ad seperately into two vectors, and then calculate cosine similarity directly from these two vectors as the relevance score. We trained CDSSM, DeepIntent and \SYS on the same training dataset, and evaluated the performance by comparing the AUCs on the same testset. 
Our results are presented in Table \ref{tb:auc}.

\begin{table}[t]
\centering
\resizebox{\columnwidth}{!}{%
\begin{tabular}{lclc|l}
  Implmentation & Decoder & Encoder Architecture & Encoder Embedding &  AUC  \\
 \hline
 (a) CDSSM & None & Conv / max pooling & Tri-letter hash  & 0.726 \\
 (b) DeepIntent & None & Conv / max pooling & Word-based   & 0.728 \\
 (c) DeepIntent & None & BLSTM / last pooling & Word-based   & 0.798\\
 (d) \SYS & Yes & BLSTM / last pooling & Word-based   & {\bf 0.840}
\end{tabular}
}
\caption{AUC scores of different Scoring Frameworks}
\label{tb:auc}
\end{table}

Note both CDSSM and DeepIntent methods only use encoders, unlike in \SYS there're both encoder and decoder components.  So to  make a fair comparison, it becomes necessary to keep the encoder setting as similar as possible, say the encoder architecture, encoded vector size, and depth of recurrent neural networks. The vector size is easy to do and we set it to 300 across the models. We also train all models with depth = 1, and set word-embedding size to 150 if applicable. Below we discuss the different encoder architectures and their AUC performance. We avoid using attention in \SYS to keep the comparison fair and simple: 

%\begin{itemize}
%\item
1) DeepIntent with BLSTM resembles \SYS the most, i.e. they have the same encoder architecture. They both start with a word-based embedding layer, leverage BLSTM to compute a sequence of hidden vectors, and take the last vector as the final encoded vector. So comparing this against \SYS can fairly show the gain by having a decoder. In Table \ref{tb:auc}, we see \SYS achieves 0.84 AUC, as shown in row (d), outperform this baseline with 0.798 AUC shown in row (c).

%\item 
2) Rather than using BLSTM, at the encoder side, CDSSM uses a convolutional (Conv) layer. The Conv layer aggregates tri-letter-based word-hash vectors via a sliding window. The output is a sequence of vectors which gets further reduced to a final encoded vector with max pooling.  In CDSSM's implemenation, it also has a fully connected layer to reduce the size of final encoded vector for online performance reason. To make a fair comparison, we set both the internal hidden vector size and final encoded vector size to 300. In Table \ref{tb:auc}, we see CDSSM implementation achieves far worse AUC score of 0.726 in row (a).

%\item 
3) With CDSSM being so different in encoder, namely the tri-letter-based embedding and the Conv layer, we modified DeepIntent implemenation to use Conv layer, in order to understand where the loss in AUC comes from. Specifically, we would like to know whether it is from the different embedding or recurrence layer. After using DeepIntent with Conv layer, the AUC of this implementation achieves only 0.728 AUC, as shown in row (b) of Table \ref{tb:auc}, similar as CDSSM implementation in row (a). It is strongly suggested, by comparing (b) against (c), that BLSTM-based encoder outperforms Conv-based encoder.

%\end{itemize}
In summary, \SYS not only provides scores allowing probabilisitc interpretation, but achieves better performance than similiarity-based scoring methods namely CDSSM and DeepIntent.
 
\subsubsection{Discussion}

We would like to briefly discuss the cost of computation and implementation, and point out why \SYS is capable of being a good relevance filter. First, from a practical point of view, having \SYS acting on top of the existing information retrieval system requires no modification to the infrastructure, minimizing implementation cost. Second, from the computation point of view, when used as a relevance scoring method, \SYS needs to calculate a score for each (query, $\text{ad}_i$) pair for all ads requesting a relevance scoring. As a result, the cost of computation grows linearly with the number of ads. This is also the reason we do not use \SYS to directly search through the entire ads database for the most relevant ones, as the hundreds of millions of ads in the database can make the search process too long to guarantee service quality. 
When used as the egress control of a information retrieval engine, however, the number of returned ads are limited; usually at the scale of tens. Moreover, batching and hierarchical softmax can further reduce the computation time required.  
%% we are encoding ads??
%Since the ad sequence are already given, we are able to avoid computing a time-consuming projection at the decoder side. As a result, a typical run takes 1 ms, which can be deployed in timing-critical scenarios like in a search engine.

\noindent
\begin{figure}
\centering
%\hspace*{-1cm}
\resizebox{\columnwidth}{!}{%
\begin{tikzpicture}[font=\sffamily,>=stealth',thick,
commentl/.style={text width=3cm, align=right},
commentr/.style={commentl, align=left},]
\node[] (user) {\LARGE User};
\node[right=3cm of user] (bot) {\LARGE Bot Agent};
\node[right=3cm of bot] (ir) {\LARGE IR System};

\draw[->] ([yshift=-0.3cm]user.south) coordinate (fin1o) -- ([yshift=-.3cm]fin1o-|bot) coordinate (fin1e) node[rectangle, pos = .5, align=center, above, sloped, draw] {Query Rewriter}
node[pos = .5, align=center, below, sloped] {{\itshape``tablet TV connector"}};

\draw[->] ([yshift=-0.5cm]bot.south) coordinate (fin3o) -- ([yshift=-0.5cm]fin3o-|ir) coordinate (fin3e) node[rectangle, pos = .5, align=center, above, sloped, draw] {Submit to IR}
node[pos = .5, align=center, below, sloped] {{\itshape``tablet TV connector"}};

\draw[->] ([yshift=-0.4cm]fin3e-|ir) coordinate (fin2o) -- ([yshift=-0.7cm]fin2o-|bot) coordinate (fin2e) node[rectangle, pos = .6, align=center, above, sloped, draw] {Returned from IR}
node[pos = .5, align=center, below, sloped] {[HDMI, micro HDMI, VGA, ...]};

\draw[->] ([yshift=-0.4cm]fin2e) coordinate (ack1o) -- ([yshift=-1.0cm]ack1o-|bot) coordinate (ack1e) 
node[rectangle, pos = .4, align=center, right, draw] {Estimate Posterior \\Distribution}
node[pos = .4, align=center, left] {[HDMI (Pr=0.5), \\micro HDMI (Pr=0.4), \\VGA (Pr=0.1), ...]};

\draw[->] ([yshift=-0.1cm]ack1e) coordinate (ack2o) -- ([yshift=-1.4cm]ack2o-|bot) coordinate (ack2e) 
node[rectangle, pos = .4, align=center, right, draw] {Information Based \\ Decision Making \\ Estimate Entropy}
node[pos = .35, align=center, left] {Confident, Recommend \\ \itshape`` Do you like this HDMI cable?"}
node[pos = .75, align=center, left, yshift=-0.4cm] {Not Confident,\\ Ask M.I. maximizing question \\ \itshape`` What size do you want?"};

\draw[->] ([yshift=-0.5cm]ack2e-|bot) coordinate (fin4o) -- ([yshift=-.3cm]fin4o-|user) coordinate (fin4e);

\draw[->] ([yshift=-0.1cm]fin4e) coordinate (ack3o) -- ([yshift=-1.0cm]ack3o-|user) coordinate (ack3e) 
node[rectangle, pos = .4, align=center, left, draw] {User Receives Recommendation\\ or User Receive Question}
node[pos = .4, align=center, right] {\itshape ``I want micro sized."};

\draw[->] ([yshift=-0.0cm]ack3e) coordinate (ack5o) -- ([yshift=-2.0cm]fin4o-|bot) coordinate (fin5e) node[rectangle, pos = .5, align=center, above, sloped, draw] {Query Rewriter}
node[pos = .5, align=center, below, sloped] {{\itshape``micro size"}};

\draw[thick, shorten >=-1cm] (user) -- (user|-ack3e);
\draw[thick, shorten >=-1cm] (bot) -- (bot|-ack3e);
\draw[thick, shorten >=-1cm] (ir) -- (ir|-ack3e);

\node[commentr, right =2mm of fin5e] {\textbf{...}};
\node[below left = 0mm and 2mm of user.south, commentl]{\textbf{INITIAL QUERY}\\[-1.5mm]{\itshape ``How do I connect tablet to TV?"}};
\node[below left = 15mm and 2mm of aux2-|user, commentl]{\textbf{...}};
\end{tikzpicture}
}
\caption{ChatBot Design}
\label{fig:chatbot}
\end{figure}
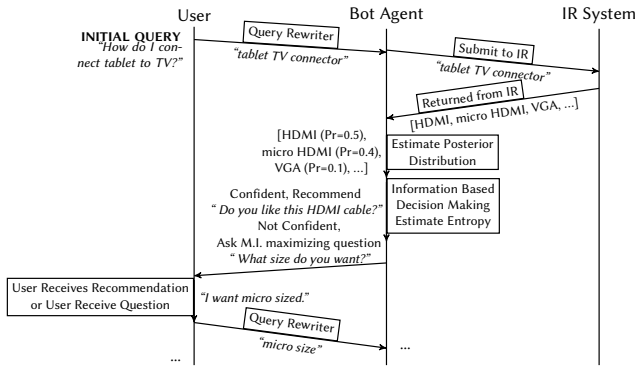

\subsection{Chatbot: Information Directed Conversation and Recommendation}\label{sec:chatbot}

The last application we introduce is a chatbot that specializes in product ad recommendation. Virtual agents and chatbots have gained popularity due to its user-friendliness and interactiveness. They not only offload some of the jobs of search engines, but also create new user interaction entry points \cite{aron2011innovative}. %, the author discussed the innovation of the Apple Siri Virual assistant. Similar services are also provided by Microsoft and Facebook, which are respectively named Cortana and M. 
Below we explain the flow of interaction with examples. %One key aspect of the framework is to formulate questions when necessary. Below we'll describe a customized implementation we have for the product recommendation domain.

%This new channel fundamentally changed human-machine interaction in multiple aspects, including multi-roundedness, information gathering and how information is presented \cite{sivaramakrishnan2007giving}.
%Traditional online targeting either uses 1) search engine, where interaction is one-shot and keyword based, or 2) user profiling, which is non-interactive and information is processed offline. As discussed above, the powerfulness and flexibility provided by virtual agents requires more powerful information processing frameworks. 

%\label{alg:algo1}
%To experiment with virtual shopping helpers, we designed and built a chatbot that specialized in ad recommendation, guided by \SYS. 

\subsubsection{Flow of Dialog and System Behaviors}
The interactive session starts with the first query submitted by a user. For example, the user can ask the chatbot ``How do I connect my tablet to TV?". The chatbot then retrieves a initial list of related ads from the information retrieval backend system. To do so, it applies \SYS's rewriting to the question and convert it into a standard query, in this case ``tablet tv connector". This standard query is then submitted to the information retrieval system which returns a list of ads, e.g. ads about HDMI cables, micro HDMI cables, or VGA cables. The chatbot then uses \SYS's posterior distribution estimation to calculate the distribution of the returned ad list, and estimates its corresponding conditional entropy. In the decision-making step, if the conditional entropy is less than some threshold $T$, the top $k$ (3 by default) most relevant ads are returned to the user. Each ad is displayed with a picture, the selling price, the merchant selling it, and embedded with a hyperlink so the user can click on. A user click will redirect the user to the e-commerce web page hosted by the merchant so the user can continue the exploration and make purchase.
Otherwise, the bot asks a conditional mutual-information maximizing question to the user, in this case the question is about the ``size'' of the connector products. For example, ``what size do you want?''. This conversation goes until a final recommendation is made. Figure \ref{fig:chatbot} gives an illustration of the procedure in the form of a timing diagram. 
Next we explain how we formulate such questions.

\subsubsection{Question Formulation}
By the principle of maximizing expected information gain, at each step, if the chatbot is not confident, it is supposed to ask a question that maximized the conditional mutual information. The problem here is, what is the set of questions we are maximizing over for? If we allow arbitrary questions, the chatbot may face issues like 1) the question may be not relevant to the product so is confusing, and 2) the mutual information is difficult to estimate. 

To address this issue, we leverage  the attributes associated with each ad. For example, an ad about a laptop has attributes ``processors", ``RAM size", ``manufacture" and so on. Similarly for clothes, there are attributes like ``color", ``size" and ``material". By formulating questions based on the attributes, the aforementioned issues go away. Firstly, it will be easier for users to relate. Users will have the perception that the chatbot is working with them to narrow down the most relevant product by confirming the attribute info. Secondly, it is straightforward to estimate mutual information along with attribute-based questions.
Notice that attributes only depend on the ads, so $\text{Input}_1^n - \text{Ad} - \text{Attribute}$
forms a Markov chain.

This allows us to estimate mutual information, as the conditional distribution, $Pr(\text{Ad}, \text{Attribute}|\text{Input}_1^n)$, can be calculated by
\[
Pr(\text{Ad}, \text{Attribute}|\text{Input}_1^n)=Pr(\text{Attribute}|\text{Ad})Pr(\text{Ad}|\text{Input}_1^n)
\label{eq:attribute}
\]
where the first factor is estimated by counting and the second factor is directly provided by \SYS's posterior update. After the information-maximizing attribute is identified, a question will be raised and the user input will be collected to update the posterior distribution again. As an example, if the user is looking for a laptop, a question may look like
\[\text{\textit{What manufacture do you like?}} \]

\subsubsection{Implementation and Qualitative Feedbacks}
We built the bot using Microsoft Bot Framework \cite{BotFramework}, which is a chatbot development tool. It supports bot conversation over various platforms, including text messages, Skype, Slack, Messenger, etc.  Figure \ref{fig:screenshot} is a screenshot of the chatbot with Skype as the platform. 

By demonstrating the prototype to a few colleagues, we got a few encouraging feedbacks. Most of them were surprised by the capability of the chatbot in recommending products when they ask related questions. The ``how to connect tablet to tv'' case was also a big win. An HDMI ad was recommended back to a user, and by clicking the ad, the title of the redirected web page popped up:  ``16.4ft Ultra-thin Micro HDMI D to A Long Cable - Connect Tablet / Smart Phone / Mobile / Laptop / Camera to HD TV''. Interested users can take this opportunity to learn more about a Micro HDMI cable (it can connect not only tablet but also other devices to TV) and purchase it! Nonetheless, several colleagues pointed out that this chatbot should not be standalone. In addition to recommending products, we should also integrate with other services to provide tutorial videos for example. Last but not least, a colleague asked whether the chatbot can provide information in other verticals other than products. By explaining how the system works, the colleague understood by training on data from a different vertical, and combining with the corresponding search engine, we can generalize the chatbot to where needed.

\begin{figure}
\centering
\includegraphics[width=0.65\linewidth]{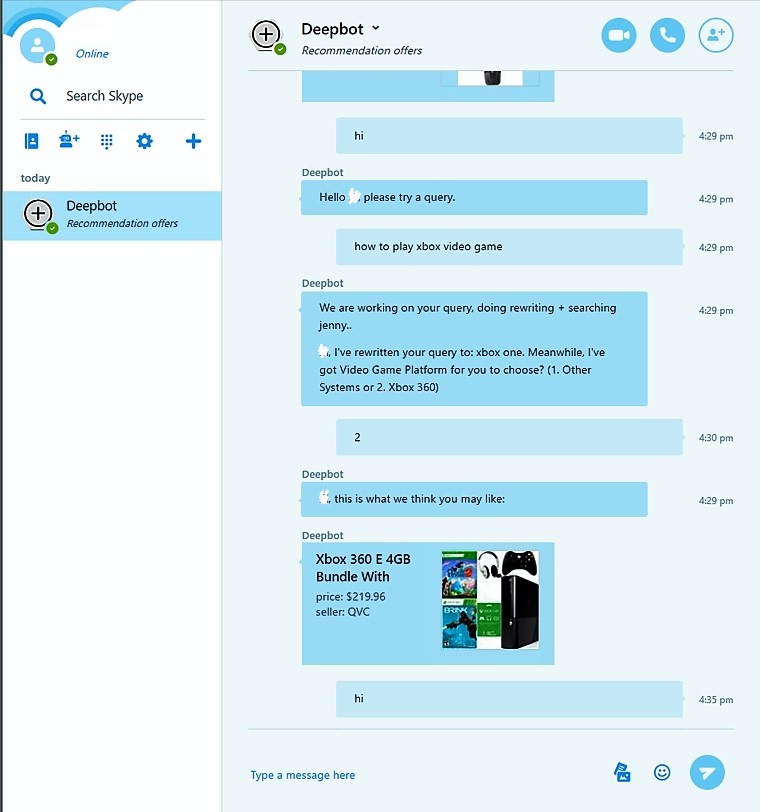}
\caption{Screenshot of the Chatbot}
\label{fig:screenshot}
\end{figure}

%% file: conclusion.tex
\section{Conclusion and Future Work}

In this paper, we introduced \SYS, a sequence-to-sequence model based framework for query understanding, ad recommendation and user interaction. In query rewriting, it significantly increases both the coverage and quality. For relevance scoring, AUC, which is a key metric, surpasses existing systems. It also demonstrated great potential in more efficient user interaction and chatbot design, for which we can rigorously formulate questions to users, based on a principle of maximizing information gain. As an ongoing work, we would like to continue work and experiment on the chatbot, possibly with quantitative experiments for the chatbot. A helpful experiment is that we can measure its efficiency (i.e. number of rounds of interaction) for a user to acquire the information he or she needs. 